\documentclass[letterpaper]{article} 
\usepackage{aaai23}  
\usepackage{times}  
\usepackage{helvet}  
\usepackage{courier}  
\usepackage[hyphens]{url}  
\usepackage{graphicx} 
\usepackage{natbib}  
\usepackage{caption} 
\usepackage{algorithm}
\usepackage{algorithmic}

\usepackage{newfloat}
\usepackage{listings}
\DeclareCaptionStyle{ruled}{labelfont=normalfont,labelsep=colon,strut=off} 
\floatstyle{ruled}
\newfloat{listing}{tb}{lst}{}
\floatname{listing}{Listing}
\usepackage{amsmath}
\usepackage{amssymb}
\usepackage{mathtools}
\usepackage{amsthm}
\usepackage[capitalize,noabbrev]{cleveref}
\usepackage{algorithm}
\usepackage{algorithmic}

\newtheorem{theorem}{Theorem}[section]

\theoremstyle{definition}
\newtheorem{definition}[theorem]{Definition}

\theoremstyle{remark}

\newtheorem{example}[theorem]{Example}

\usepackage{tikz}
\usetikzlibrary{automata,positioning}
\usetikzlibrary{arrows.meta}

\newcommand{\wt}{\widetilde}

\newcommand{\hkappa}{h^{(\kappa)}}
\newcommand{\thetakappa}{\theta^{(\kappa)}}

\newcommand{\kappatheta}{\kappa^{(\theta)}}

\newcommand{\algTPI}{TLPI}

\newcommand{\algQPI}{\text{QLPI}}
\newcommand{\algQDQN}{\text{QL-DQN}}
\newcommand{\appxVstar}{$\wt V^\star$}

\newcommand{\calM}{\mathcal{M}}
\newcommand{\calS}{\mathcal{S}}
\newcommand{\calA}{\mathcal{A}}
\newcommand{\calX}{\mathcal{X}}
\newcommand{\calY}{\mathcal{Y}}

\newcommand{\bbE}{\mathbb{E}}
\newcommand{\bbR}{\mathbb{R}}

\DeclareMathOperator*{\argmax}{arg\,max}

\newcommand{\costh}[1]{c ( #1 )}
\title{Planning and Learning with Adaptive Lookahead}
\author{
    Aviv Rosenberg\textsuperscript{\rm 1}\thanks{Research conducted while the author was an intern at Nvidia Research.}, Assaf Hallak\textsuperscript{\rm 2}, Shie Mannor\textsuperscript{\rm 2,3}, Gal Chechik\textsuperscript{\rm 2,4}, Gal Dalal\textsuperscript{\rm 2}
}
\affiliations{
    \textsuperscript{\rm 1} Amazon, \textsuperscript{\rm 2} Nvidia Research, 
    \textsuperscript{\rm 3} Technion, 
    \textsuperscript{\rm 4} Bar-Ilan University \\
    avivros@amazon.com
}

\begin{document}

\maketitle

\begin{abstract}
Some of the most powerful reinforcement learning frameworks use planning for action selection. Interestingly, their planning horizon is either fixed or determined arbitrarily by the state visitation history. Here, we expand beyond the naive fixed horizon and propose a theoretically justified strategy for adaptive selection of the planning horizon as a function of the state-dependent value estimate. We propose two variants for lookahead selection and analyze the trade-off between iteration count and computational complexity per iteration. We then devise a corresponding deep Q-network algorithm with an adaptive tree search horizon. We separate the value estimation per depth to compensate for the off-policy discrepancy between depths. Lastly, we demonstrate the efficacy of our adaptive lookahead method in a maze environment and Atari.
\end{abstract}

\section{Introduction}

The celebrated Policy Iteration (PI) \cite{howard1960dynamic} and Value Iteration (VI) \cite{sutton2018reinforcement} algorithms are the basis for most state-of-the-art reinforcement learning (RL) algorithms. Since both PI and VI are based on a one-step greedy approach for policy improvement, so are the most commonly used policy-gradient \cite{schulman2017proximal, haarnoja2018soft} and Q-learning \cite{mnih2013playing, hessel2018rainbow} based approaches. In each iteration, they perform an improvement of their current policy by looking one step forward and acting greedily. While this is the simplest and most common paradigm, stronger performance was recently achieved using multi-step lookahead. Notably, in AlphaGo \citep{silver2018general} and MuZero  \cite{schrittwieser2020mastering}, the multi-step lookahead is implemented via Monte Carlo Tree Search (MCTS) \citep{browne2012survey}. In MCTS, the search depth is not chosen adaptively but gradually increases with the aggregation of state visitations.

\begin{table}[t]
    \begin{center}
        \begin{tabular}[b]{|c|c|c|c|c|}
            \hline
            Algorithm & \#Iterations & Iteration complexity
            \\
            \hline \hline
            PI \cite{scherrer2016improved}& $1$ & $ c(1)$
            \\
            \hline
               H-PI &  $\frac{1}{H}$ &  $c(H)$
               \\
            \cite{efroni2018beyond} & &
            \\
            \hline
            TLPI (this paper)& $\frac{1}{h^{(\kappa)}}$ & $c(1) + \theta^{(\kappa)} c(h^{(\kappa)})$
            \\
            \hline
            QLPI (this paper)& $\frac{\log \gamma}{\log \kappa} \left(\leq  \frac{1}{h^{(\kappa)}}\right)$ & $O \left( \sum_{h=1}^H \theta_h c(h) \right)$
            \\
            \hline
        \end{tabular}
    \end{center}
    \caption{Algorithm comparison summary. The iterations count (number of iterations) is divided by $S (A-1) \frac{\log (1-\gamma)}{\log \gamma}$. The iteration complexity (computational complexity per iteration) is divided by $S$. $c(k)$ is the computational complexity of $k$-step planning. $h^{(\kappa)}$ is the smallest integer $h$ s.t. $\gamma^{h^{(\kappa)}} \le \kappa$. $\theta^{(\kappa)}$ is the fraction of $\kappa$-contracting states. $\theta_1,\dots,\theta_H$ are the contraction quantiles sizes.}
    \label{table:regret bounds}
\end{table}

Several recent works rigorously analyzed the properties of multi-step lookahead in common RL schemes \citep{efroni2018beyond,moerland2020think,hallak2021improve,sikchi2022learning}. These and other related literature studied a fixed planning horizon chosen {\em in advance}. However, both in simulated and real-world environments, a large variety of states benefit differently from various lookahead horizons. A grasping robot far from its target will learn very little from looking a few steps into the future, but if the target is within reach, much more precision and planning are required to grasp the object correctly. Similarly, at the beginning of a chess game, lookahead grants little information about which move is better, while agents in mid-game intricate situations benefit immensely from considering all future possibilities for the next few moves. 
Indeed, in this work, we devise a  methodology for adaptive selection of the planning horizons in each state and show it achieves a significant speed-up of the learning process. 

We propose two complementing approaches to determine the suitable horizon per state in each PI iteration. To do so, we keep track of the room for improvement for the value function. Our first algorithm, Threshold-based Lookahead PI (\algTPI{}), ensures the desired convergence rate and minimizes the computational complexity for each iteration. Alternatively, our second algorithm, Quantile-based Lookahead PI (\algQPI{}), takes the per-iteration computational complexity as a given budget and aims for the best possible convergence rate. We prove that both \algTPI{} and \algQPI{} converge to the optimum and achieve a significantly lower computational cost than their fixed-horizon alternative (see Table~\ref{table:regret bounds}).

Next, we devise \algQDQN{}: a DQN \cite{mnih2013playing} variant of \algQPI{}, where the policy chooses an action by employing an exhaustive tree search \cite{hallak2021improve}  looking $h$ steps into the future. The tree-depth $h$ is chosen adaptively per state to achieve an overall improved convergence rate at a reduced computational cost. To sustain on-policy consistency while generalizing over the multiple depths, we use a different value network per depth, where the first layers are shared across networks. We test our method on Atari and show it improves upon a fixed-depth tree search.

To summarize, our contributions are the following. First, we propose to use adaptive state-dependent lookahead and devise corresponding algorithms. Our analysis shows they converge with improved computational complexity. Second, we extend our approach to online learning with a DQN variant that uses an exhaustive tree search of adaptive depth and per-depth value network. Third, we evaluate the proposed methods on maze and Atari environments and show better results than a fixed lookahead horizon.

\section{Preliminaries}
\label{sec:preliminaries}

We consider a discounted MDP $\calM = (\calS, \calA, P, r,\gamma)$, where $\calS$ is a finite state space of size $S$, $\calA$ is a finite action space of size $A$, $r: \calS \times \calA \to [0,1]$ is the reward function, $P: \calS \times \calA \to \Delta_\calS$ is the transition function, and $\gamma \in (0,1)$ is the discount factor.
Let $\pi: \calS \to \calA$ be a stationary policy, and $V^\pi \in \bbR^S$ be the value function of $\pi$ defined by $V^\pi(s) = \bbE \left[ \sum_{t=0}^\infty \gamma^t r(s_t,\pi(s_t)  \mid s_0 = s \right]$.

The goal of a planning algorithm is to find the optimal policy $\pi^\star$ such that, for every $s \in \calS$,
\[
    V^\star(s)
    =
    V^{\pi^\star}(s)
    =
    \max_{\pi: \calS \to \calA} V^\pi(s).
\]
Given a policy $\pi$, let $T^\pi: \bbR^S \to \bbR^S$  be the Bellman operator:
$
    T^\pi [V] = r^\pi + \gamma P^\pi V,
$
where $r^\pi(s) = r(s,\pi(s))$ and $P^\pi(s' | s) = P(s' | s,\pi(s))$.
It is well known that the value of policy $\pi$ is the unique solution to the linear equations: $T^\pi [V^\pi] = V^\pi$.
Let $T: \bbR^S \to \bbR^S$ be the optimal Bellman operator defined as:
\[
    T[V](s)
    =
    \max_{a \in \calA} r(s,a) + \gamma \sum_{s' \in \calS} P(s' | s,a) V(s').
\]
Then, the optimal value is the unique solution to the nonlinear equations $T[V^\star] = V^\star$ and $T$ is a $\gamma$-contraction in the max-norm over the state space:
\[
\lVert V^\star - T[V^{\pi}] \rVert_\infty
\le
    \gamma \lVert V^\star - V^{\pi} \rVert_\infty.
\]

\subsection{PI and $h$-PI}
PI starts from an arbitrary policy $\pi_0$ and performs iterations that consist of: (1) an evaluation step that evaluates the value of the current policy, and (2) an improvement step that performs a 1-step improvement based on the computed value. That is, for $n=0, 1, 2, \dots$,
\[
    \pi_{n+1}(s)
    =
    \argmax_{a \in \calA} r(s,a) + \gamma \sum_{s' \in \calS} P(s' \mid s,a) V^{\pi_n}(s').
\]
By the contraction property of the Bellman operator, one can prove that PI finds the optimal policy after at most ${\left\lceil (\log \frac{1}{\gamma})^{-1} S (A - 1) \log \frac{1}{1 - \gamma} \right\rceil}$ iterations \citep{scherrer2016improved}.

The PI algorithm can be extended to $h$-PI by performing $h$-step improvements (instead of $1$-step).
Formally, define the $Q$-function of policy $\pi$ with a $h$-step lookahead as 
\[
    Q^\pi_h(s,a) = \max_{ \{ \pi_t \}_{t=1}^{h-1} } \bbE^{s,a} \left[ \sum_{t=0}^{h-1} \gamma^t r(s_t,\pi_t(s_t)) + \gamma^h V^{\pi}(s_h) \right],
\]
where $\bbE^{s,a} [\cdot] = \bbE [\cdot | s_0=s,\pi_0(s) = a]$.
Then, the update rule of $h$-PI is
$\pi_{n+1}(s) = \argmax_{a \in \calA} Q^{\pi_n}_h(s,a)$.
The operator induced by $h$-step lookahead is a $\gamma^h$-contraction which allows to reduce a factor of $h$ from the bound on the number of iterations until convergence \citep{efroni2018beyond}, i.e., it is bounded by $\left\lceil ( h \log \frac{1}{\gamma})^{-1} S (A - 1) \log \frac{1}{1 - \gamma} \right\rceil$.

Multi-step lookahead guarantees that the number of iterations to convergence is smaller than the $1$-step lookahead, but it comes with a computational cost.
Computing the $h$-step improvement may take exponential time in $h$.
In tabular MDPs, this can be mitigated with the use of dynamic programming \citep{efroni2020online}, while in MDPs with large (or infinite) state space, MCTS \cite{browne2012survey} or the alternative exhaustive tree-search \cite{hallak2021improve} are used in forward-looking fashion. To compare our algorithms, in the rest of the paper we measure the computational complexity as follows:

\begin{definition}
    \label{def:comp-cost-h}
    Let $\costh{h}$ be the computational cost of performing a $h$-step improvement in a single state.
    For example, in a deterministic full $A$-ary tree we have $\costh{h} = O( A^h ).$
\end{definition}

\section{Motivating Example}
\label{sec:motivation}

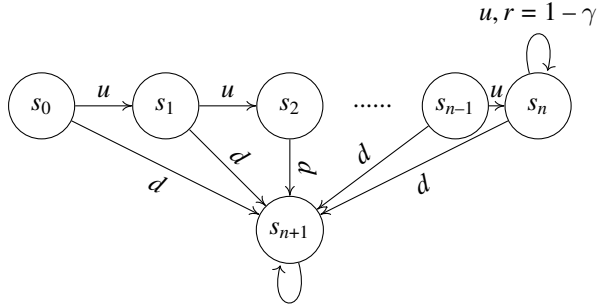
\begin{figure}[t]
    \centering
\begin{tikzpicture}[shorten >=0pt,node distance=0.75cm,on grid,auto,scale=0.4,every node/.style={scale=0.8}] 
    	\node[state] (s_0) at (0,0)  {$s_{n+1}$}; 
    	\node[state] (s_1) at (-6,3) {$s_0$}; 
    	\node[state] (s_2) at (-3,3) {$s_1$};
    	\node[state] (s_3) at (0,3) {$s_2$};
    	\node[state] (s_n) at (4,3) {$s_{n-1}$};
    	\node[state] (s_S) at (7,3) {$s_n$}; 
    	\node at (2,3) {......};
    	\path[-{>[scale=1.5,width=2.5]}]
    	(s_0) edge [loop below] node {$ $} ()
    	(s_S) edge [loop above] node {$u,r=1-\gamma$} ()
    	(s_1) edge	node [below,sloped,inner sep=3pt] {$d$} (s_0)
    	(s_1) edge	node [above,sloped,inner sep=3pt] {$u$} (s_2)
    	(s_2) edge	node [above,sloped,inner sep=3pt] {$d$} (s_0)
    	(s_2) edge	node [above,sloped,inner sep=3pt] {$u$} (s_3)
    	(s_3) edge	node [above,sloped,inner sep=3pt] {$d$} (s_0)
    	(s_n) edge	node [above,sloped,inner sep=3pt] {$u$} (s_S)
    	(s_n) edge	node [above,sloped,inner sep=3pt] {$d$} (s_0)
    	(s_S) edge	node [below,sloped,inner sep=3pt] {$d$} (s_0);
    \end{tikzpicture}
    \caption{Chain MDP example with deterministic transitions and rewards 0 everywhere except for $r(s_n, u) = 1 - \gamma.$ Using a fixed horizon $h$ per state for each PI iteration leads to the same convergence rate as using horizon $\ell$ in only a single state for $\ell=2,\dots,h$, but at a much higher computational cost.}
    \label{fig:convergence-rate-example}
\end{figure}

To show the potential of our approach, consider the chain MDP example in \cref{fig:convergence-rate-example}:
\begin{example} [Chain MDP]
\label{exmp: chain mdp}
Let $\mathcal{M}$ be an MDP with $n+1$ states $s_0,s_1,\dots,s_n$ and a single sink state $s_{n+1}$. 
Each of the $n+1$ states transitions to the consecutive state with action $u,$ and to the sink state with action $d.$
All rewards are $0$ except for state $s_n$ in which $u$ yields reward $1 - \gamma$.
\end{example}

Now consider the standard PI algorithm initialized with $\pi_0(s_i) = d$ for all $i \in \{0, \dots, n\}$.
Since the reward at the end of the chain needs to propagate backward, in each iteration the value of only a single state is updated. Thus, PI takes exactly $n$ iterations to converge to the optimal policy $\pi^\star(s_i) = u$ for all $i$. Instead, with a fixed horizon $h=2$, the reward propagates through two states in each iteration (instead of one) and therefore convergence takes $\lceil n/2 \rceil$ iterations. Generally, PI with a fixed horizon $h$, i.e., $h$-PI, converges in  $\lceil n/h \rceil$ iterations.

While $h$-PI converges faster (in terms of iterations) as $h$ increases, in most states, performing $h$-step lookahead does not contribute to the speed-up at all. 
In our example, we can achieve exactly the same convergence rate as $2$-PI by using a $2$-step lookahead in only a single state in each iteration (and $1$-step in all other states).
Specifically, we need to pick the state that is exactly $2$ steps behind the last updated state in the chain.
For general $h$, consider applying $\ell$-step lookahead in only one state --- the one that is $\ell$ steps behind the last updated state in the chain --- for $\ell=2,\dots,h$ and $1$-step in the others. This guarantees the same number of iterations until convergence as $h$-PI, but with much less computation time.
Namely, while the per-iteration computational cost of $h$-PI is $O \bigl( n \cdot \costh{h} \bigr)$, we can achieve the same convergence rate with just $O \bigl( n \cdot \costh{1} + \sum_{\ell=2}^h \costh{\ell} \bigr).$ 
In practice, when $n$ is large and $\costh{h}$ scales exponentially with $h$, this gap can be immense: $O \bigl( n \cdot 2^h \bigr)$ versus $O \bigl( n + 2^h \bigr)$.

\section{Contraction-Based Adaptive Lookahead}
\label{sec:approach}

In this section, we introduce the concept of dynamically adapting the planning lookahead horizon during runtime, based on the online obtained contraction.
In \cref{exmp: chain mdp}, $h$-PI convergence rate can be achieved when using a lookahead larger than $1$ in just $h$ states.
The question is how to choose these states? In the example, the chosen states are evidently those with the maximal distance between their $1$-step improvement and optimal value, i.e., 
$
    \argmax_{s \in \calS} |V^\star(s) - T[V^{\pi_t}](s)|.
$
In this section, we show that this approach also leads to theoretical guarantees on the convergence of the PI algorithm.

To understand how the convergence rate depends on the distance of the $1$-step improvement from the approximate optimal value, we delve into the theoretical properties of PI. 
Since the standard $1$-step improvement yields a contraction of $\gamma$ while the $h$-step improvement gives $\gamma^h$, $h$-PI converges $h$ times faster than standard PI \cite{efroni2018beyond}. Importantly, this contraction is with respect to the $L_\infty$ norm; i.e., the states with the worst (largest) contraction coefficient determine the convergence rate of PI.
This behavior is the source of weakness of using a fixed lookahead. Example~\ref{exmp: chain mdp} shows that one state may slow down convergence, but it also hints at an elegant solution: {\em use larger lookahead in states with larger contraction coefficient}. 
 
We leverage this observation and present two new algorithms: \algTPI{} which aims to achieve a fixed contraction in all states with a reduced computational cost, and \algQPI{} which aims to achieve maximal contraction in every iteration within a fixed computational budget. While both algorithms seek to optimize a similar problem, their analyses differ and shed light on the problem from different perspectives: \algTPI{} depends on the contraction factor per state, while \algQPI{} considers only the ordering of the states with respect to their contraction factors.

Since we do not have access to the optimal value $V^\star(s)$, the algorithms rely on warm-starts or an approximation of the optimal value $V^\star$, denoted by \appxVstar{}. While obtaining a good approximation of the value function is hard, we aim at a simpler task: find an approximation that is informative for allocating the depth resources. The approximated values may be far from optimal, as long as they yield similar allocation across depths.
In many cases, we can obtain such an approximation through, e.g., state aggregation, training agents on similar tasks, or by running PI for a small number of iterations. 
In Sections~\ref{sec:experiments} and~\ref{sec:atari}, we show that these approximation methods are indeed effective in practice.

\subsection{Threshold-based Lookahead Policy Iteration}
\label{sec:thresh-alg}

\begin{algorithm}[t]
    \caption{\algTPI{}}  
    \label{alg:threshold-PI}
    \begin{algorithmic}[1]
        \STATE \textbf{Input:} $\mathcal{S} , \mathcal{A} , r , P , \gamma ,\kappa, \beta, \wt V^\star$.
        
        \STATE \textbf{Initialization:} 
        Arbitrary $\pi_0$, $t \gets 0$.
        
        \WHILE{$\pi_t$ changes}
        
            \STATE Evaluation: compute $V^{\pi_t}$, and set $U(s,a) \gets \infty.$
        
            \STATE $1$-step improvement: $U(s,a) \gets Q^{\pi_t}_1(s,a) \  \forall (s,a).$
            
            \STATE $\hkappa$-step improvement: $U(s,a) \gets Q^{\pi_t}_{\hkappa}(s,a)$ for every $(s,a)$ s.t.: (here $U(s) = \max_{a} U(s,a)$)
            \begin{align}
                \label{eq:TPI-alg}
                | \wt V^\star(s) - U(s) | > \kappa \lVert \wt V^\star - V^{\pi_t} \rVert_\infty - \beta.
            \end{align}
            
            \STATE   Set $\pi_{t+1}(s) \gets \argmax_{a \in \calA} U(s,a)$ for every $s \in \calS.$
            
        \ENDWHILE
        
    \end{algorithmic}
\end{algorithm}

\algTPI{} (\cref{alg:threshold-PI}) takes as input the approximated value $\wt V^\star$, a desired contraction factor $\kappa$ and a correction term $\beta.$
We assume that $\lVert V^\star - \wt V^\star \rVert_\infty \le \epsilon$.
This implies we can measure contraction up to some approximation error that scales with $\epsilon$.
The algorithm ensures that in each iteration, the value in every state contracts by at least $\kappa$. 
This is achieved by first performing $1$-step improvement in all states and then performing $\hkappa$-improvement in states whose measured contraction is less than $\kappa$, where $\hkappa$ is the smallest integer $h$ such that $\gamma^h \le \kappa$.
Since we do not have an accurate estimate of the optimal value, we use the correction term $\beta$ to make sure that no states falsely seem to achieve the desired contraction due to the approximation error $\epsilon$ (see \cref{eq:TPI-alg}).

The following result states that \algTPI{} converges at least as fast as $h$-PI \cite{efroni2018beyond} with $h$ set to $\hkappa-1,$ and with improved computational complexity. To measure the trade-off between the contraction factor (that determines the convergence rate) and the computational complexity needed to achieve it, \cref{def:contraction-coefficient} presents $\thetakappa$ as the fraction of states in which we perform a large lookahead.

\begin{definition}[Def. of $\thetakappa$]
    \label{def:contraction-coefficient}
    Let $\{\pi_t\}_{t=1}^T$ be the sequence of policies generated by \algTPI{}  with correction term $\beta$ and approximated value \appxVstar{}. Let $\kappa \in (0,1),$ and define
    \begin{equation*}
    {\calX_t = \{ s \in \calS : | \wt V^\star(s) - T[V^{\pi_t}](s) | \leq \kappa \lVert \wt V^\star - V^{\pi_t} \rVert_\infty - \beta \}} 
    \end{equation*}
    as the set of states, that after $1$-step improvement in iteration $t$, are $\beta$-close to be contracted by $\kappa$ with respect to $\wt V^\star$.
    Then, denote by $\thetakappa = \max_{1\leq t \leq T} |\calS \setminus \calX_t| / S$ the largest fraction of states with contraction less than $\kappa,$ observed along all policy updates. 
\end{definition}

\begin{theorem}
\label{thm:threshold-PI}
    The \algTPI{} algorithm with approximated value $\wt V^\star$ and correction term $\beta = \epsilon (\kappa + 1)$ converges in at most $\left\lceil \left( (\hkappa - 1) \log \frac{1}{\gamma}\right)^{-1} S (A - 1) \log \frac{1}{1 - \gamma} \right\rceil$ iterations.
    Moreover, its per-iteration computational complexity is bounded by ${S \cdot \Bigl( \costh{1} + \thetakappa \costh{\hkappa} \Bigr)}$.
\end{theorem}

\begin{proof}[Proof sketch]
    The proof to bound the number of iterations follows \citet{scherrer2016improved} while utilizing two key observations.
    First, the convergence analysis of PI only uses the contraction property of the Bellman operator w.r.t. $V^\star$, and not w.r.t. an arbitrary pivot vector.
    The distance to $V^\star$ can be approximated using \appxVstar{}, and the approximation error is handled by the correction term $\beta$.
    Second, by the construction of the algorithm, a contraction of at least $\kappa$ in every state is guaranteed.
    The computational complexity follows because we perform $1$-step lookahead in all states and $\hkappa$-step lookahead in $\thetakappa$ of the states by \cref{def:contraction-coefficient}. For the complete proof, see Appendix~\ref{sec: TPI full proof}.
\end{proof}

To illustrate the merits of \algTPI{} and Thoerem~\ref{thm:threshold-PI}, consider the Chain MDP in \cref{exmp: chain mdp} where we set $\kappa=\gamma^h$ for some $h \in \mathbb{N}$ (and assume $\wt V^\star = V^\star$ for simplicity). In every iteration, the states not contracted by $\kappa$ after $1$-step improvement are the $h$ states closest to the end of the chain that have not been updated yet (recall that each state reaches the correct optimal value after just one non-idle update).
Thus, $\thetakappa = h/S$ and the per-iteration computational complexity is $S \cdot \costh{1} + h \cdot \costh{h}$.

\subsection{Quantile-based Lookahead Policy Iteration}

\begin{algorithm}[t]
    \caption{\algQPI{}}  
    \label{alg:quantiles-PI}
    \begin{algorithmic}[1]
        \STATE \textbf{Input:} $\mathcal{S} , \mathcal{A} , r , P , \gamma , (\theta_1 \dots \theta_{H}), m, \wt V^\star$.
        
        \STATE \textbf{Initialization:} 
        Arbitrary $\pi_0$, $t \gets 0$.
        
        \WHILE{$\pi_t$ changes}
        
            \STATE Evaluation: compute $V^{\pi_t}$, and set $U(s,a) \gets \infty.$
            
            \FOR{$h=1,2,\dots,H$}
            
                \STATE Compute $q_h$ as the $(1 - \theta_{h} - m/S)$ quantile of ${\{ | \wt V^\star(s) - \max_a U(s,a) | \}_{s \in \calS}}$.
            
                \STATE $h$-step improvement: $U(s,a) \gets Q^{\pi_t}_h(s,a)$ for every $(s,a)$ s.t.:
                $
                    | V^\star(s) - \max_{a \in \calA} U(s,a) | \ge q_h.
                $
                
            \ENDFOR
            
        \STATE Set $\pi_{t+1}(s) \gets \argmax_{a \in \calA} U(s,a)$ for every $s \in \calS.$
            
        \ENDWHILE
    \end{algorithmic}
\end{algorithm}

\algQPI{} (\cref{alg:quantiles-PI}) resembles \algTPI, but instead of a contraction coefficient $\kappa,$ it takes as input a vector of quantiles (budgets)  $(\theta_1,\dots,\theta_H) \in [0,1]^H$ for some predetermined maximal considered lookahead $H$. 
Instead of the actual distance to the optimal value, \algQPI{} relies only on the ordering of the states in terms of distance from the optimum.
This allows for weaker requirements on the approximated value \appxVstar{} as it should only preserve the order.

\begin{definition}
    Let $p_s$ and $\tilde p_s$ be the positions of state $s$ in the orderings of $\{ | V^\star(s) - V^{\pi_t}(s)| \}_{s \in \calS}$ and $\{ |\wt V^\star(s) - V^{\pi_t}(s)| \}_{s \in \calS}$, respectively.
    We define the approximation \appxVstar{} to be $m$\emph{-order-preserving} if, for every $s \in \calS$, $|p_s - \tilde p_s| \le m$.
\end{definition}

\emph{State-aggregation} is an example of an approximation that preserves orders and that is available in many domains where states are based on locality (like the maze environment considered in \cref{sec:experiments}).
Assume that we have access to a state-aggregation scheme that splits the state space into $S/m$ groups of size $m$ such that for every two states $s_1,s_2$ in the same group $|V^\star(s_1) - V^\star(s_2)| \le \epsilon$ and for any state $s_3$ from a different group $|V^\star(s_1) - V^\star(s_3)| > 2 \epsilon$.
Then the optimal value of the aggregated MDP $V^\star_{agg}$ is $m$-order-preserving as long as $| V^\star_{agg}(s) - V^\star(s)| \le \epsilon$ for every $s \in \calS$, since the position of a state can be shifted due to the aggregation by at most the size of its group $m$.

\algQPI{} attempts to maximize the contraction in every iteration while using $\ell$-step lookahead in at most $\theta_{\ell} \cdot S + m$ states. This is achieved by performing $\ell$-step improvement on the $(\theta_{\ell} + m/S)$ portion of states that are furthest away from \appxVstar{}.

The following result is complementary to Theorem~\ref{thm:threshold-PI}: now, instead of choosing the desired iteration complexity (via $\kappa$ in \algTPI), we choose the desired computational complexity per iteration via budgets $(\theta_1,\dots,\theta_H).$ For the resulting iteration complexity we define the induced contraction factor:
\begin{definition} [Def. of $\kappatheta$]
    \label{def:contraction-coefficient-quantiles}
    Let $\{\pi_t\}_{t=1}^T$ be the sequence of policies generated by \algQPI{}.
    Let $h_t^\theta(s)$ be the largest lookahead applied in state $s$ in iteration $t$ when running \algQPI{} with quantiles $(\theta_1,\dots,\theta_H)$.
    For a given $\kappa$, define
    \begin{equation*}
    {\calY_{t} (\kappa) = \{ s: | V^\star(s) - T^{h_t^\theta(s)}[V^{\pi_t}](s) | \leq \kappa \lVert V^\star - V^{\pi_t} \rVert_\infty \}}    
    \end{equation*}
    as the set of states contracted by  $\kappa$ in iteration $t$.
    The induced contraction factor $\kappatheta$ is defined as the minimal $\kappa$ such that $\calY_t(\kappa) = \calS$ for every  $t$.
\end{definition}

Though its formal definition may seem complex, $\kappatheta$ is simply the effective contraction obtained by \algQPI{}.

\begin{theorem}
    \label{thm:quantiles-PI}
    The \algQPI{} algorithm converges in at most $\left\lceil (\log \frac{1}{\kappatheta})^{-1} S (A - 1) \log \frac{1}{1 - \gamma} \right\rceil$ iterations.
    Moreover, its per-iteration computational complexity is bounded by ${S \cdot \sum_{h=1}^{H} (\theta_{h} + m/S) \costh{h}}$.
\end{theorem}
We provide the proof in Appendix~\ref{sec: QPI full proof}; it is based on similar ideas as the proof of Theorem~\ref{thm:threshold-PI}.

To illustrate the merits of \algQPI{} and Theorem~\ref{thm:quantiles-PI}, consider the Chain MDP in \cref{exmp: chain mdp} where we set $\theta_1=1$ and $\theta_2=\dots=\theta_h=1/S$ for some $h \in \mathbb{N}$ (again $\wt V^\star = V^\star$ for simplicity). In every iteration \algQPI{} first performs $1$-step lookahead in all states, and then, for each $\ell=2,\dots,h$, it performs $\ell$-step lookahead in exactly one state -- the state that is $\ell$ steps behind the last updated state in the chain.
As explained in \cref{sec:motivation}, the induced contraction is $\kappatheta = \gamma^h$ and \algQPI{} converges in $n/h$ iterations with \emph{optimal} per-iteration complexity of $S \cdot \costh{1} + \sum_{\ell=2}^h \costh{\ell}$.

Finally, we highlight the complementary nature of the two algorithms: while in \algTPI{} the complexity parameter $\thetakappa$ is governed by the desired contraction coefficient, in \algQPI{} the induced contraction $\kappatheta$ is the outcome of the pre-determined computational budget.

\begin{figure}[t]
    \centering 
    \includegraphics[width=120pt,height=120pt]{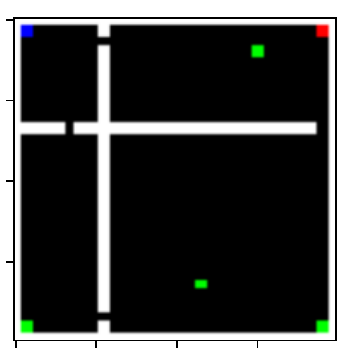}
    \caption{Snapshot of our maze environment, a $30 \times 30$ grid world. The agent is spawned in the top left corner (blue) and needs to reach one of four randomly chosen goals (green), while avoiding the trap (red). White pixels denote walls. Upon reaching a goal state, the agent re-appears in a new random location. }
  \label{fig:maze}
\end{figure}

\section{Maze Experiments}
\label{sec:experiments}

In the first set of experiments, we evaluate our adaptive lookahead algorithms, \algTPI{} and \algQPI{}, on a grid world with walls \citep{tennenholtz2022covariate}.
Specifically, we used a $30 \times 30$ grid world that is divided to four rooms with doors between them; see \cref{fig:maze}.
The agent is spawned in the top left corner (blue) and needs to reach one of four randomly chosen goals (green) where the reward is $1$, while avoiding the trap (red) that incurs a reward of $-1$.
There are four deterministic actions (up, down, right, left). Upon reaching a goal, the agent is moved to a random state.
We set $\gamma = 0.98.$

\paragraph{Fixed lookahead.}
We begin with testing the fixed-horizon $h$-PI with values $h=1,2,\dots,7$. To corroborate that larger lookahead values reduce the number of PI iterations required for convergence, in \cref{fig:fixed-lookahead-num-iterations-paper}, we show the distance from the solution as the function of iteration for the different depths. The plot demonstrates the effect of the lookahead in a less pathological example than \cref{exmp: chain mdp}. 

In \cref{fig:num-queries}, we compare the overall computational complexity, and not only the number of iterations, of the different fixed lookahead values. To measure performance, we count the number of \emph{queries to the simulator} (environment) until convergence to the optimal value. More efficient lookahead horizons will require fewer overall calls to the simulator.

Beginning with the fixed lookahead results in the leftmost plot, we see the trade-off when picking the lookahead. For a lookahead too short (1 in this case), the convergence requires too many iterations such that even the low computational complexity of each iteration is not sufficient to compensate for the total compute time. Note that $h$-PI with $h=1$ is the standard PI algorithm, which evidently performs worse than the best-fixed lookahead although it is overwhelmingly the most widely used version of PI.
On the other extreme of a very large lookahead, each iteration is too computationally expensive, despite the smaller number of iterations.

\begin{figure}[t]
    \centering 
    \includegraphics[width=170pt,height=120pt]{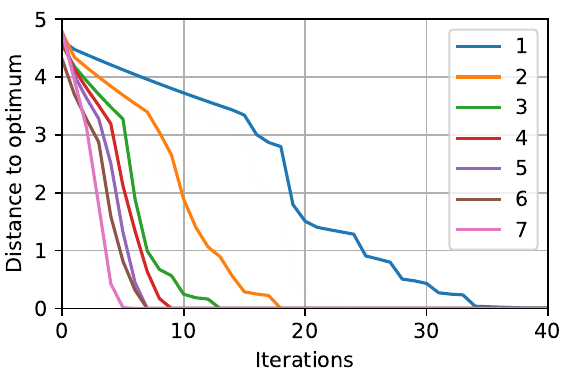}
    \caption{The distance to the optimum,  captured by $\lVert V^\star - V^{\pi_t} \rVert_\infty,$ as a function of the iteration number $t$ for fixed lookahead values of $h=1,2,\dots,7$ in maze environment.}
  \label{fig:fixed-lookahead-num-iterations-paper}
\end{figure}

\begin{figure*}[t]
    \centering 
    \includegraphics[width=\textwidth,height=120pt]{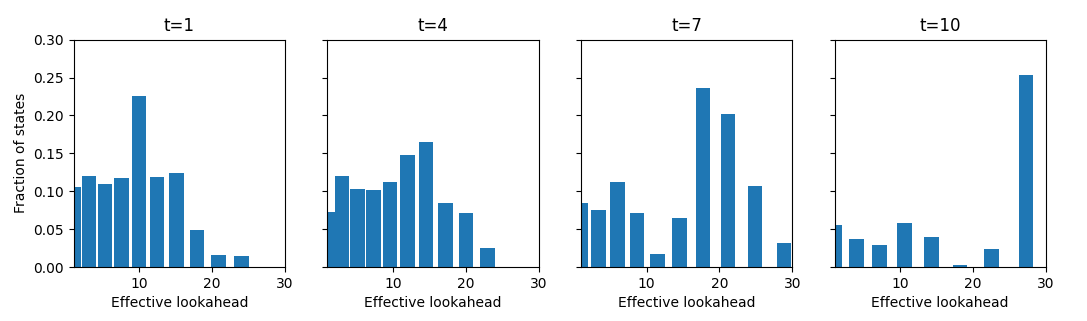}
    \caption{Histograms for the fraction of states per effective lookahead along several iterations of PI. The effective lookahead of contraction factor $\kappa$ is defined as $h = \log_{\gamma} (\kappa)$, i.e., $\gamma^h = \kappa$.}
  \label{fig:hist-alternative}
\end{figure*}

\begin{figure*}[t]
    \centering 
    \includegraphics[width=\textwidth,height=130pt]{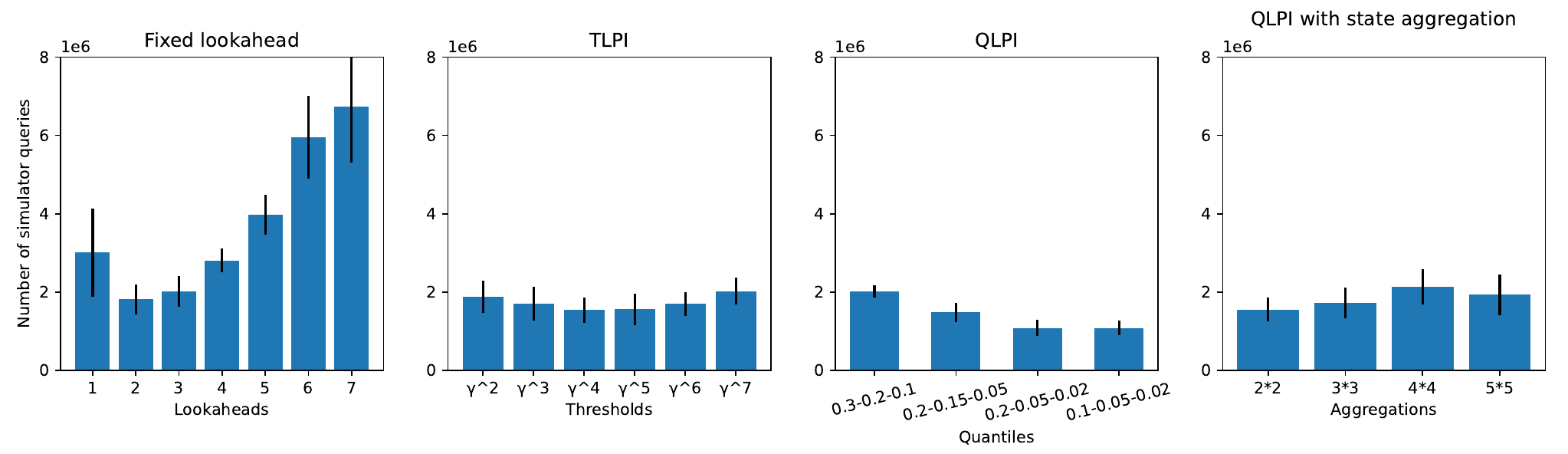}
    \caption{Number of queries to the simulator until convergence. Lower is better. 
    The results are averaged across $10$ runs and the error bars represented standard deviation. \algQPI{} is run with lookaheads $1$, $2$, $4$, and $8$, where the quantiles in the x-axis represent $\theta_2,\theta_4,\theta_8$.}
  \label{fig:num-queries}
\end{figure*}

\begin{figure*}[t]
    \centering 
    \includegraphics[width=\textwidth]{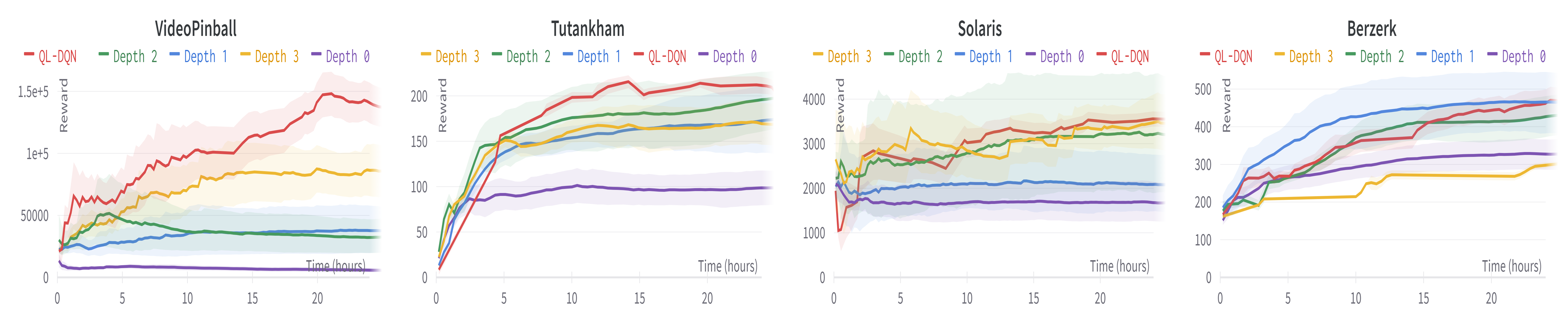}
    \caption{Average and std of training reward of \algQDQN{} (in red) and DQN with fixed tree-depths $0,1,2,3$ in various Atari environments. Note the x-axis is time due to the increasing computational complexity of higher depths (a plot with step-based x-axis is given in Appendix \ref{supp:plots}).}
  \label{fig:atari-experiments}
\end{figure*}

\paragraph{\algTPI{}.}
To verify our observation that a long lookahead is wasteful in large parts of the state space, we first plot a histogram of the contraction factor along several PI iterations in \cref{fig:hist-alternative}. Here we see that indeed the effective contraction factor $\kappa$ is much smaller than $\gamma$ (i.e., more contractive than $1$-step lookahead) in roughly 90\% of the states. 

Next, we run \algTPI{} with  $\kappa = \gamma^2,\gamma^3,\dots,\gamma^7$ and an accurate approximated value $\wt V^\star = V^\star$. The results are given in \cref{fig:num-queries}, second plot. By \cref{thm:threshold-PI}, when setting $\kappa = \gamma^h$, we expect the same number of iterations until convergence as $h$-PI but with better computational complexity. In fact, the results reveal even stronger behavior: \algTPI{}($\gamma^h$) for all $h=1,2,\dots,7$ achieves similar computational complexity compared to the \emph{best} fixed lookahead witnessed in $h$-PI.

\paragraph{\algQPI{}.}
In all our experiments we run \algQPI{} with $\theta_1 = 1$ and $\theta_3=\theta_5=\theta_6=\theta_7=0$ (again $\wt V^\star = V^\star$).
For $(\theta_2,\theta_4,\theta_8)$ we set the following values: $(0.3,0.2,0.1)$, $(0.2,0.15,0.05)$, $(0.2,0.05,0.02)$ and $(0.1,0.05,0.02),$ which respectively depict decreasing weights to depths $2, 4, 8.$
The results are presented in \cref{fig:num-queries}, third plot.
Again we can see that for all the parameters, \algQPI{} performs as well as the best-fixed lookahead.
Moreover, notice that for some of the choices of $\theta$ vectors, the performance significantly improves upon the best-fixed horizon.

\paragraph{Approximate $V^\star$ via state aggregation.}
We again run \algQPI{} with budget values $(0.1,0.05,0.02)$, but replace  $V^\star$ with an approximation we obtain with state aggregation.
Namely, we merge squares of $k \times k$ into a single state, solve the smaller aggregated MDP, and use its optimal value as an approximation for $V^\star$. 

We perform this experiment with $k=2,3,4,5$ and include the aggregated MDP solution process in the total simulator query count. This way, our final algorithm does not have any prior knowledge of $V^\star.$ The results are presented in \cref{fig:num-queries}, last plot. As expected, the performance is slightly worse than the original \algQPI{} that uses the accurate $V^\star$, but for all different aggregation choices the algorithm still performs as well as the best-fixed lookahead in $h$-PI. 

To summarize, the maze experiments show that with adaptive planning lookahead, we manage to reach the solution with better sample complexity (i.e. number of simulator queries) compared to fixed-horizon $h$-PI. More importantly, our methods are robust to hyperparameter choices: the improved results are obtained uniformly with \emph{all} various tested parameters of \algTPI{} and \algQPI{}. This alleviates the heavy burden of finding the best-fixed horizon for a given environment.

\section{QL-DQN and Atari Experiments}
\label{sec:atari}

\begin{figure}[t]
    \centering
\includegraphics[width=0.48\textwidth,height=168pt]{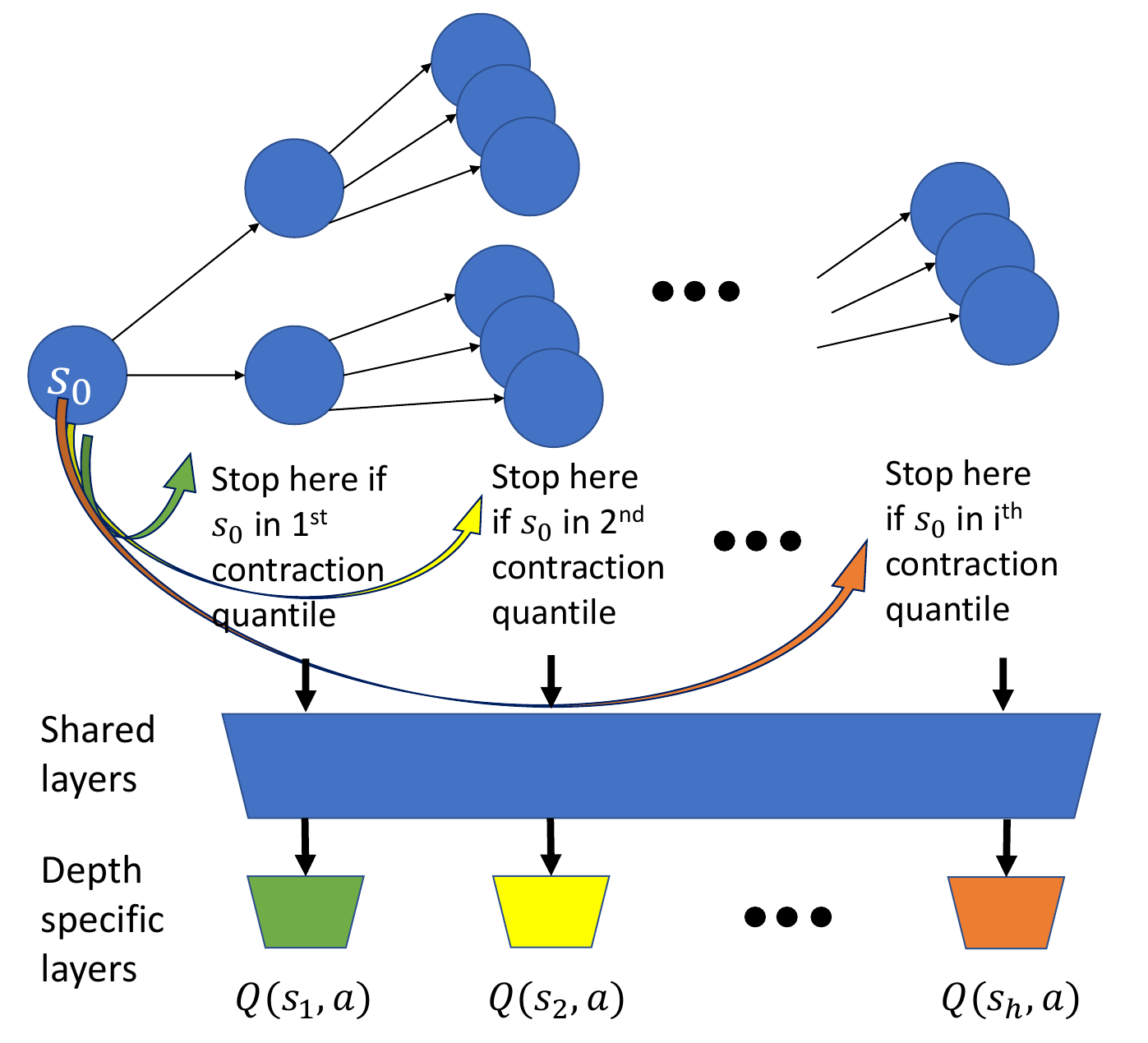}
    \caption{The \algQDQN{} algorithm. When choosing actions, the policy uses a tree depth based on the ranking of $s_0$'s contraction coefficient in the replay buffer. The per-depth $Q$ network has a shared basis and depth-specific heads.}  \label{fig:QL-DQNalg}
\end{figure}

In this section, we extend our adaptive lookahead algorithm \algQPI{} to neural network function approximation. We present Quantile-based lookahead DQN (\algQDQN{}): the first DQN algorithm that uses state-dependent lookahead that is dynamically chosen throughout the learning process. \algQDQN{} is visualized in Figure~\ref{fig:QL-DQNalg} and operates as follows:
\begin{enumerate}
    \item We introduce a \emph{per-depth Q-function}. Technically, we maintain $H$ parallel $Q$-networks (where $H$ is the maximal tree depth) and use the $h$-th network when predicting the value of a leaf in depth $h$ in the tree. To improve generalization and data re-use, the $H$ networks share the initial layers (feature extractors).

\item Maintain a sorted replay buffer according to the distance between the current value estimate and approximated optimal value (for efficiency we utilize a priority queue). For $\tilde{V}^\star,$ we train a standard DQN (depth $0$) agent for only 1M steps -- a relatively quick process.

\item For a replay buffer of size $N$, maintain $H$ quantiles based on the above ordering. The quantile sizes are $\theta_1 \cdot N, \dots, \theta_H \cdot N$. The values $\theta_i$ are hyper-parameters.

\item Choosing tree depth: Per state during simulation, start by spanning the 1st level of the tree. If the best leaf value distance from the approximate optimal value is in the 1st quantile, end the tree-search. Otherwise, go one level deeper, compare to the 2nd quantile, and re-iterate with the same logic. Continue possibly until the max depth $H$ is reached. Notice that the tree-search is feasible in reasonable run-time thanks to highly efficient parallel Atari simulation on GPU \citep{hallak2021improve}.

\item Choosing action given depth: After spanning the tree as described above, return the first action (from the root) that corresponds to the highest leaf value.

\item Store (state, action, cumulative reward to leaf, leaf state, tree-depth) to the replay buffer.

\item For training, set the bootstrap target to be the cumulative reward from the tree-search plus the Q-value corresponding tree-depth calculated on the leaf-state.
\end{enumerate}

We note that the per-depth Q-function is crucial in order to keep online consistency and achieve convergence. 
We found that in practice, the Q-function learned by DQN is not the true cumulative reward of rollouts. Instead, it is some function that minimizes the Bellman error. This phenomenon is orthogonal to our work and was also recently studied in \cite{fujimoto2022should}. In our context of multiple-depth Q-network, if we bootstrap using the target from one depth for another, the above phenomenon causes inconsistencies that lead to divergence. To handle this inconsistency, we introduced the multi-head Q-network for multiple depths and found that it solves the issue.
All other parts of the algorithm and hyper-parameter choices are taken as-is from the original DQN paper \cite{mnih2013playing}. 

We train QL-DQN on several Atari environments \citep{bellemare2013arcade}. Since our work aims to improve sample complexity over fixed-horizon baselines, our metric of interest here is the reward as a function of training time. Hence, in \cref{fig:atari-experiments} we present the convergence of \algQDQN{} versus DQN with fixed depths $0$ through $3$, as a function of time. The plots consist of the average score across $5$ seeds together with std values. Note that depth $0$ corresponds to standard DQN (the baseline). As seen, QL-DQN achieves better performance on VideoPinball and Tutankham, while on Solaris and Berzerk, it is on par with the best-fixed lookahead. 

The conclusion here is again that we obtain a better or similarly-performing agent to a pre-determined fixed planning horizon. This comes with the benefit of robustness to the expensive hyper-parameter choice of the best-fixed horizon per a given environment.

\section{Discussion}

In this paper we propose the first planning and learning algorithms that dynamically adapt the multi-step lookahead horizon as a function of the state and the current value function estimate.
We demonstrate the significant potential of adaptive lookahead both theoretically --- proving convergence with improved computational complexity, and empirically --- demonstrating their favorable performance in a maze and Atari. Our algorithms often perform as well as the best-fixed horizon in hindsight in almost all the experiments, while in some cases they surpass it. Future work warrants an investigation whether the best-fixed horizon can always be outperformed by an adaptive horizon.

Theoretically, our guarantees rely on prior knowledge of an approximate optimal value, raising the question whether one can choose lookahead horizons adaptively without any prior knowledge, e.g., using transfer learning based on similarity between domains.
Moreover, when the forward model performing the lookahead is inaccurate or learned from data, the adaptive state-dependent lookahead itself may serve as a quantifier for the level of trust in the value function estimate (short lookahead) versus the model (long lookahead). This can offer a way for state-wise regularization of the learning or planning problem.
Our work is also related to the growing Sim2Real literature. In particular, when having several simulators with different computational costs and fidelity levels. The lookahead problem then translates to choosing in which states to use which simulator with which lookahead.

Our focus in this paper was reducing iteration and overall complexity; we thus ignored more intricate details of the forward search itself. Additional practical aspects such as CPU-GPU planning efficiency trade-offs \cite{hallak2021improve} can also affect the lookahead selection problem. One promising direction is to expand at each step only the few most promising nodes, and keep the search width fixed after a certain value. This gives linear complexity in the search depth instead of exponential, at the risk of missing relevant paths.

\bibliography{refs}

\newpage
\onecolumn
\appendix

\section{Proofs}

\subsection{Proof of Theorem~\ref{thm:threshold-PI}}
\label{sec: TPI full proof}
\paragraph{Bounding the number of iterations to convergence.}
Let $B_t:\bbR^S \to \bbR^S$ be the operator induced by the algorithm, i.e., for state $s$ it is $T$ if $|\wt V^\star(s) - T[V^{\pi_t}](s)| \le \kappa \lVert \wt V^\star - V^{\pi_t} \rVert_\infty - \beta$, and otherwise it is  $T^{\hkappa}$. 
The \algTPI{} algorithm ensure that, for every iteration $t$ and state $s \in \calS$,
\[
    | B_t[V^{\pi_t}](s) - V^\star(s) |
    \le
    \kappa \lVert V^\star - V^{\pi_t} \rVert_\infty,
\]
and therefore: $\lVert B_t[V^{\pi_t}] - V^\star \rVert_\infty \le \kappa \lVert V^\star - V^{\pi_t} \rVert_\infty$.
We now show that the sequence $\{ \lVert V^\star - V^{\pi_t} \rVert_\infty \}_{t \ge 0}$ is contracting with coefficient $\kappa$.
To show this, we split the states into two groups:
\begin{enumerate}
    \item $| T[V^{\pi_t}](s) - \wt V^\star(s) | \le \kappa \lVert \wt V^\star - V^{\pi_t} \rVert_\infty - \beta$.
    Thus, \algTPI{} uses $1$-step lookahead for $s$, which implies:
    \begin{align*}
        V^\star(s) - V^{\pi_{t+1}}(s)
        & =
        V^\star(s) - \wt V^\star(s) + \wt V^\star(s) -   T[V^{\pi_{t}}](s) 
        \\
        & \quad 
        + T[V^{\pi_{t}}](s) - T^{\pi_{t+1}}[V^{\pi_{t}}](s) + T^{\pi_{t+1}}[V^{\pi_{t}}](s) - T^{\pi_{t+1}}[V^{\pi_{t+1}}](s)
        \\
        & \le
        V^\star(s) - \wt V^\star(s) + \wt V^\star(s) -   T[V^{\pi_{t}}](s)
        \\
        & \quad + T^{\pi_{t+1}}[V^{\pi_{t}}](s) - T^{\pi_{t+1}}[V^{\pi_{t+1}}](s)
        \\
        & \le
        \epsilon + \kappa \lVert \wt V^\star - V^{\pi_t} \rVert_\infty  - \beta + \gamma P^{\pi_{t+1}} (V^{\pi_t} - V^{\pi_{t+1}})
        \\
        & \le
        \epsilon + \kappa \lVert \wt V^\star - V^\star \rVert_\infty + \kappa \lVert V^\star - V^{\pi_t} \rVert_\infty - \beta
        \\
        & \le
        \kappa \lVert V^\star - V^{\pi_t} \rVert_\infty + (\kappa+1) \epsilon - \beta
        =
        \kappa \lVert V^\star - V^{\pi_t} \rVert_\infty,
    \end{align*}
    where the second step follows since $T^{\pi_{t+1}}[V^{\pi_{t}}](s) \ge T[V^{\pi_{t}}](s)$, the third step is by the definition of $\epsilon$ and the state $s$ in the first group, the forth step is by monotonicity of PI, and the last step is by definition of $\beta$.
    
    \item $| T[V^{\pi_t}](s) - \wt V^\star(s) | > \kappa \lVert \wt V^\star - V^{\pi_t} \rVert_\infty - \beta$.
    Thus, \algTPI{} uses $\hkappa$-step lookahead for $s$, which implies:
    \begin{align*}
        V^\star(s) - V^{\pi_{t+1}}(s)
        & =
        T T^{\hkappa - 1} [V^\star] (s) - T T^{\hkappa - 1}[V^{\pi_{t}}](s) 
        \\
        & \quad + 
        T T^{\hkappa - 1} [V^{\pi_{t}}](s) - T^{\pi_{t+1}} T^{\hkappa - 1} [V^{\pi_{t}}](s) 
        \\
        & \quad + 
        T^{\pi_{t+1}} T^{\hkappa - 1}[V^{\pi_{t}}](s) - T^{\pi_{t+1}}[V^{\pi_{t+1}}](s)
        \\
        & \le
        T T^{\hkappa - 1} [V^\star] (s) - T T^{\hkappa - 1}[V^{\pi_{t}}](s) 
        \\
        & \quad + 
        T^{\pi_{t+1}} T^{\hkappa - 1}[V^{\pi_{t}}](s) - T^{\pi_{t+1}}[V^{\pi_{t+1}}](s)
        \\
        & \le
        \gamma \lVert T^{\hkappa - 1} [V^\star] - T^{\hkappa - 1}[V^{\pi_{t}}] \rVert_\infty 
        \\
        & \quad + 
        \gamma P^{\pi_{t+1}} (T^{\hkappa - 1}[V^{\pi_{t}}] - V^{\pi_{t+1}})
        \\
        & \le
        \gamma \lVert T^{\hkappa - 1} [V^\star] - T^{\hkappa - 1}[V^{\pi_{t}}] \rVert_\infty
        \\
        & \le
        \gamma^{\hkappa} \lVert V^\star - V^{\pi_{t}} \rVert_\infty
        \le
        \kappa \lVert V^\star - V^{\pi_{t}} \rVert_\infty,
    \end{align*}
    where the second step follows since $T^{\pi_{t+1}} T^{\hkappa - 1}[V^{\pi_{t}}](s) \ge T T^{\hkappa - 1}[V^{\pi_{t}}](s)$, the forth step is by monotonicity of PI, the fifth step is since $T^{\hkappa-1}$ is $\gamma^{\hkappa-1}$-contracting (because $T$ is $\gamma$-contracting), and the last step is by definition of $\hkappa$.
\end{enumerate}

We now follow the proof of \citet{scherrer2016improved} for bounding the number of iterations of PI, and introduce the notation $A^{\pi'}_\pi = T^{\pi'}[V^\pi] - V^\pi$.
Then,
\begin{align*}
    \lVert A^{\pi_t}_{\pi^\star} \rVert_\infty
    & =
    \lVert T^{\pi_t}[V^\star] - V^\star \rVert_\infty
    \le
    \lVert V^\star - V^{\pi_t} \rVert_\infty
    \le
    \kappa^t \lVert V^\star - V^{\pi_0} \rVert_\infty
    \le
    \frac{\kappa^t}{1 - \gamma} \lVert A^{\pi_0}_{\pi^\star} \rVert_\infty,
\end{align*}
where the first and last inequalities are by \citet{scherrer2016improved}, and the second inequality is because the sequence is contracting with coefficient $\kappa$.
By the definition of the max-norm, and as a $A^{\pi_0}_{\pi^\star} \le 0$ (using the fact that $\pi^\star$ is optimal), there exists a state $s_0$ such that $- A^{\pi_0}_{\pi^\star}(s_0) = \lVert A^{\pi_0}_{\pi^\star} \rVert_\infty$.
We deduce that for all $t$,
\[
    - A^{\pi_t}_{\pi^\star}(s_0)
    \le
    \lVert A^{\pi_t}_{\pi^\star} \rVert_\infty
    \le
    \frac{\kappa^t}{1 - \gamma} \lVert A^{\pi_0}_{\pi^\star} \rVert_\infty
    =
    - \frac{\kappa^t}{1 - \gamma} A^{\pi_0}_{\pi^\star} (s_0).
\]
As a consequence, the action $\pi_t(s_0)$ must be different from $\pi_0(s_0)$ when $\frac{\kappa^t}{1 - \gamma} < 1$, that is for all values
of $t$ satisfying
\[
    t
    \ge 
    t^\star
    =
    \left\lceil \frac{\log \frac{1}{1 - \gamma}}{\log \frac{1}{\kappa}} \right\rceil.
\]
In other words, if some policy $\pi$ is not optimal, then one of its non-optimal actions will be eliminated for good after at most $t^\star$ iterations. 
By repeating this argument, one can eliminate all non-optimal actions
(there are at most $S (A - 1)$ of them), and the result follows.

\paragraph{Bounding the per-iteration computational complexity.}
In every iteration $t$, we first perform $1$-step improvement in all of the states.
This has a computational cost of $O(S \cdot \costh{1})$.
By definition of $\thetakappa$, after performing $1$-step improvement in all of the states, there are at most $\thetakappa S$ states that are not contracted by at least $\kappa$ (with respect to \appxVstar{} and after subtracting the correction term $\beta$).
Thus, by definition of the \algTPI{} algorithm, we perform $\hkappa$-step improvement in at most $\thetakappa S$ states.
This has a computational cost of $O(\thetakappa S \cdot \costh{\hkappa})$.

\subsection{Proof of Theorem~\ref{thm:quantiles-PI}}
\label{sec: QPI full proof}
\paragraph{Bounding the number of iterations to convergence.}
The proof follows the same path as the proof of \cref{thm:threshold-PI}, but now the operator $B_t$ induced by the algorithm is $\kappatheta$-contracting instead of $\kappa$-contracting.
Since \appxVstar{} is $m$-order-preserving, the fact that we use the quantiles $(\theta_1 + m/S,\dots,\theta_H + m/S)$ instead of $(\theta_1,\dots,\theta_H)$ ensures that indeed $\kappatheta$ contraction is obtained in all the states.

\paragraph{Bounding the per-iteration computational complexity.}
By definition of the \algQPI{} algorithm, in every iteration $t$, we perform $h$-step lookahead in at most $\theta_h S + m$ states.
For every $h=1,\dots,H$, the computational complexity of performing $h$-step lookahead in $\theta_h S + m$ states is at most $O \left( (\theta_h S + m) \cdot \costh{h} \right)$, which gives a total computational complexity of $O \left(S \cdot \sum_{h=1}^H (\theta_h + m/S) \cdot \costh{h} \right)$ per iteration.

\section{Additional Plots}\label{supp:plots}

\begin{figure*}[h]
    \centering 
    \includegraphics[width=\textwidth]{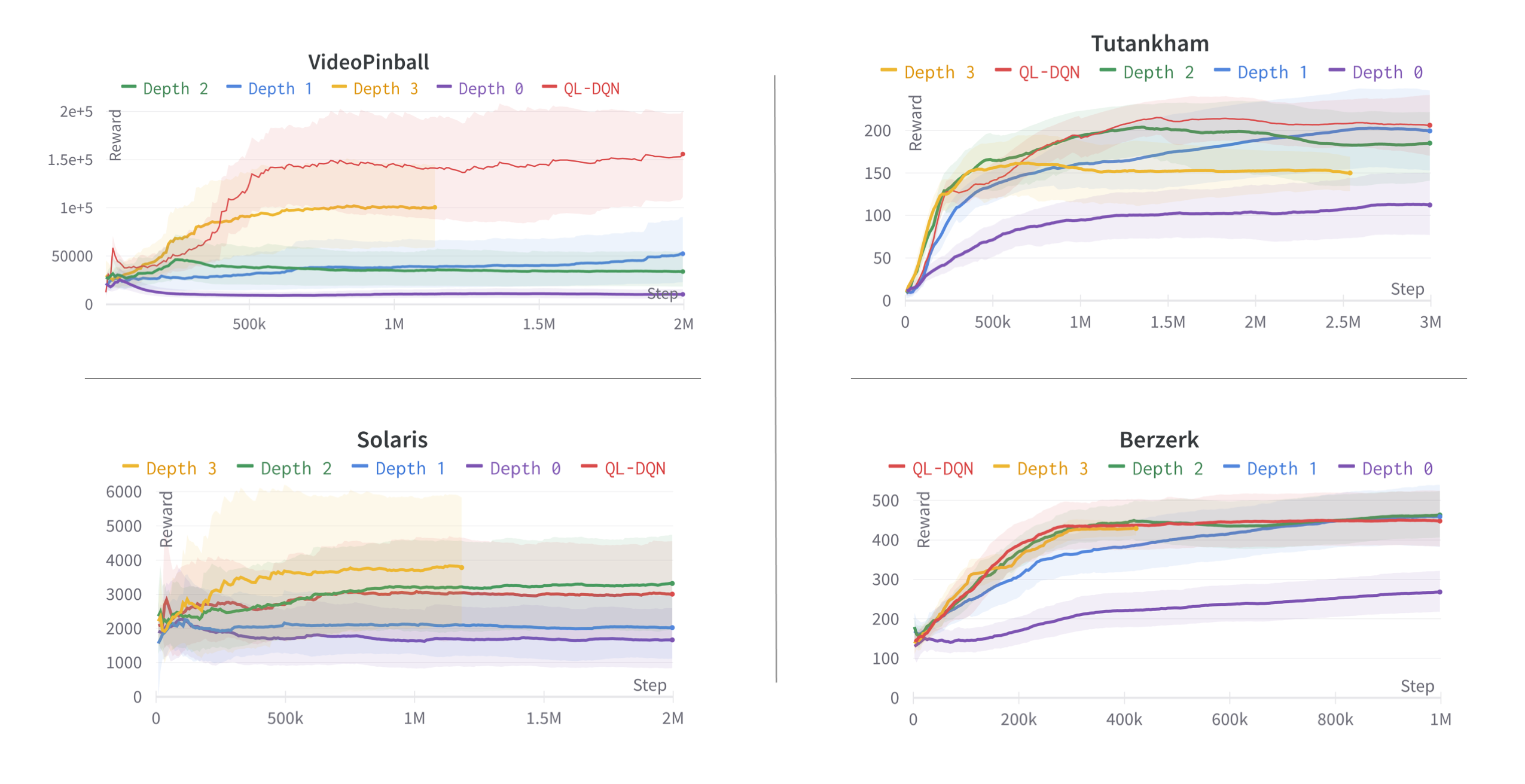}
    \caption{Figure \ref{fig:atari-experiments} where the x-axis is number of steps instead of training time. As expected, the larger the horizon, the better the performance in terms of environment steps (as opposed to wallclock time). Tutankham with depth 3 is the one exception here, which we found to be less stable (and hence the large variance). }
  \label{fig:atari-experiments_bystep}
\end{figure*}

\end{document}